\documentclass[11pt]{arxform}
\usepackage{stmaryrd}

\usepackage{xcolor}

\usepackage{amsmath, amscd, amsfonts, amssymb, graphicx, xcolor, textcomp}
\usepackage[bookmarksnumbered, colorlinks, plainpages]{hyperref}
\usepackage{subfig, epstopdf, algorithmic, algorithm,nicefrac}
 \usepackage{placeins}

\renewcommand{\vec}[1]{\boldsymbol{#1}}

\newcommand*{\defeq}{\mathrel{\vcenter{\baselineskip0.5ex \lineskiplimit0pt
			\hbox{\footnotesize.}\hbox{\footnotesize.}}}%
	=}

\begin{document}
	\title{Machine Learning with the Sugeno Integral:\\ The Case of Binary Classification}

	\author{Sadegh Abbaszadeh and Eyke H{\"u}llermeier}
	
	\institute{Paderborn University\\
		Heinz Nixdorf Institute and Department of Computer Science\\
		Intelligent Systems and Machine Learning Group\\
		\email{sadegh.abbaszadeh@uni-paderborn.de, eyke@upb.de}}

\maketitle

	\begin{abstract}
In this paper, we elaborate on the use of the Sugeno integral in the context of machine learning. More specifically, we propose a method for binary classification, in which the Sugeno integral is used as an aggregation function that combines several local evaluations of an instance, pertaining to different features or measurements, into a single global evaluation. Due to the specific nature of the Sugeno integral, this approach is especially suitable for learning from ordinal data, that is, when measurements are taken from ordinal scales. This is a topic that has not received much attention in machine learning so far. The core of the learning problem itself consists of identifying the capacity underlying the Sugeno integral. To tackle this problem, we develop an algorithm based on linear programming. The algorithm also includes a suitable technique for transforming the original feature values into local evaluations (local utility scores), as well as a method for tuning a threshold on the global evaluation. To control the flexibility of the classifier and mitigate the problem of overfitting the training data, we generalize our approach toward $k$-maxitive capacities, where $k$ plays the role of a hyper-parameter of the learner. 
We present experimental studies, in which we compare our method with competing approaches on several benchmark data sets. \\[2mm]
	\end{abstract}

\section{Introduction}

The idea of combining models and aggregation functions from the field of (multi-criteria) decision making with data-driven approaches for model identification from the field of machine learning has attracted increasing attention in recent years. Examples of such combinations include methods for learning the majority rule model \cite{SobrieMP13}, the non-compensatory sorting model \cite{SobrieMP15}, or the Choquet integral \cite{fallhullermei12}. In contrast to many other machine learning approaches, corresponding models are interpretable and meaningful from of decision making point of view, a property that has gained increasing attention in the recent past \cite{demirel2018choquet,abbaszadeh2018fuzzy}. Besides, they often guarantee other properties that might be desirable, such as monotonicity. 

The general structure of such models, sketched in Fig.\ \ref{fig:mcml}, is as follows: Given a choice alternative described in terms of an attribute vector $\vec{x} = (x_1, \ldots , x_m)$, each attribute $x_i$ is first evaluated by means of a local utility function, and thereby turned into a local utility degree $u_i = u_i(x_i) \in \mathbb{R}$\,---\,it corresponds to what is called a ``criterion'' in multi-criteria decision analysis. In a second step, the local utility degrees $u_1, \ldots , u_m$ are aggregated into a global utility $U = U(u_1, \ldots , u_m)$. Finally, a decision or an action $y$ is taken based on this utility. We refer to this setting as \emph{multi-criteria machine learning} (MCML).

Our work is motivated by recent contributions, in which the aggregation step is accomplished by means of the (dicrete) Choquet integral \cite{fallhullermei12,TehraniH13,mpub256,mpub244}. As a versatile aggregation function, the Choquet integral has a number of properties that are quite appealing from a machine learning point of view. For example, it allows for combining non-linearity with monotonicity, i.e., to model nonlinear yet monotone dependencies between input attributes (criteria) and outcomes (global utilities, decisions). Besides, it is able to capture complex interactions between different input attributes. 

In this paper, we consider the MCML setting with the Sugeno integral \cite{suge_to} instead of the Choquet integral as an aggregation function. The former can be seen as the qualitative counterpart of the latter: As it operates on a purely ordinal (instead of a numerical) scale, it appears to be specifically suitable for learning from ordinal data. Moreover, whereas machine learning methods based on the Choquet integral can be seen as generalizations of learning linear models, methods based on the Sugeno integral are conceptually similar to symbolic, logic-based model classes such as decision trees and rule-based models, in which predictions are obtained by testing properties of the attribute values $x_i$ of an instance $\vec{x}$, typically comparing them with certain thresholds, but without doing any numerical computations with them.

In the next section, we recall the definition of the Sugeno integral and some of its basic properties.  In Section~\ref{VC}, we introduce a class of binary classifiers, which are based on thresholding the Sugeno integral, and analyze the flexibility of this model class in terms of its VC dimension. A method for learning binary classifiers of that kind is introduced in Section~\ref{learning}. Section~\ref{Complexity} addresses the idea of controlling the degree of maxitivity of the Sugeno integral as a means for adjusting the model flexibility to the complexity of the data, thereby avoiding the problem of overfitting the training data. Section~\ref{Experimental} presents the results of an experimental study, prior to concluding the paper in Section \ref{Conclusion}. 

\begin{figure}[t]
\begin{center}
\includegraphics[width=9cm]{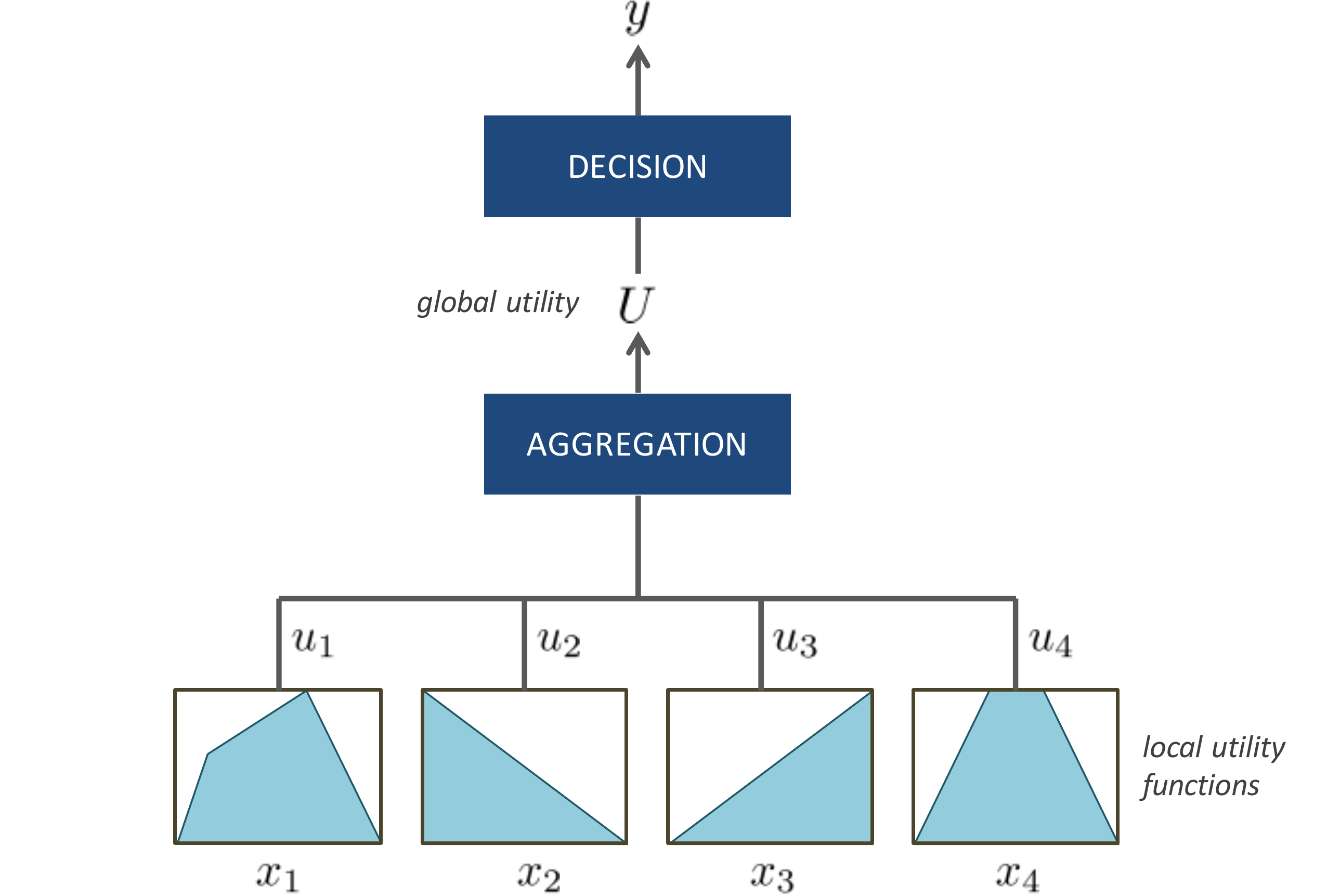}
\caption{Structure of a multi-criteria machine learning model.}
\label{fig:mcml}
\end{center}
\end{figure}

\section{The Sugeno Integral}
\label{Sugeno}

Recall that a set function on $[m] \defeq \{1, \ldots , m\}$ is a real-valued function defined on all subsets of $[m]$. A capacity on $[m]$ is a set function $ \mu: 2^{[m]} \rightarrow \mathbb{R}$ that preserves monotonicity, i.e., such that $\mu(A) \leq \mu(B)$ for all $A, B \in 2^{[m]}$ with $A \subseteq B$. The capacity $\mu$ is normalized if $\mu([m])=1$, and additive if $\mu(A \cup B) = \mu(A) + \mu(B)$ for any disjoint subsets $A, B \subseteq [m]$. Likewise, it is called ``maxitive'' if $\mu(A \cup B) = \mu(A) \vee \mu(B) = \max(\mu(A), \mu(B))$ for any subsets $A, B \subseteq [m]$  \cite[p.\ 173]{grabisch2009aggregation}.



\subsection{Definition of the Sugeno Integral}

The Sugeno integral is an aggregation function defined with respect to a capacity, i.e., it combines a set of values $u_1, \ldots , u_m$ pertaining to $m$ criteria into a single representative value, accounting for the importance of subsets of criteria as specified by the capacity. Since the values $u_i$ are measured on an ordinal scale, the Sugeno integral relies on disjunctive and conjunctive operations (instead of addition and multiplication). 
Formally, it can be expressed in various ways \cite{grabisch2009aggregation}. A common definition of the Sugeno integral with respect to the capacity $\mu$ on $[m]$ is as follows:
\begin{equation}\label{1}
S_{\mu}(\vec{u})= S_{\mu}(u_1, \ldots , u_m) = \bigvee_{j=1}^{m} \left(u_{\sigma(j)} \wedge \mu\left(A_{\sigma(j)}\right)\right),
\end{equation}
where $\sigma$ is a permutation on $[m]$ such that $u_{\sigma(1)} \leq u_{\sigma(2)} \leq \ldots \leq u_{\sigma(m)}$ and  $A_{\sigma(j)} = \{\sigma(j), \sigma(j+1), \ldots, \sigma(m)\}$. Since the Sugeno integral is a particular weighted lattice polynomial function, its disjunctive normal representation is given by
\begin{equation}\label{2}
S_{\mu}(\vec{u})= \bigvee_{A \subseteq [m]} \left(\bigwedge_{j \in A} u_j \wedge \mu(A)\right) \enspace .
\end{equation}
Considering disjunction and conjunction as mathematical formalizations of existential and universal quantification, respectively, this representation suggests the following interpretation: The Sugeno integral $S_{\mu}(\vec{u})$ is high if (and only if) there exists a subset of criteria $A$, such that $A$ has high importance and all values $u_i$ on these criteria are high. In other words, $\mu(A)$ is a measure of sufficiency of the criteria $A$: Satisfying all criteria in $A$ (i.e., achieving high utilities) is enough to achieve an overall high utility.

Another important representation of the Sugeno integral, which we will exploit later on, is a definition in terms of a median \cite[p.\ 213]{grabisch2009aggregation}:
\begin{equation*} 
S_{\mu}(\vec{u})= \text{Med}\Big(u_1, \ldots, u_m, \mu\left(A_{\sigma(2)}\right),  \ldots, \mu\left(A_{\sigma(m)}\right)\Big)
\end{equation*}
Thus, the Sugeno integral can be obtained by sorting the $2m -1$ values, which are given as arguments to $\text{Med}$ in the above expression, from smallest to largest, and then taking the value at position $m$ in this sorted list.

\subsection{The $k$-maxitive Sugeno Integral}

With a measure (capacity) $\mu$, interactions and dependencies between criteria can be modeled in a very flexible way. An obvious drawback, however, is the exponential complexity implied by this approach: The specification of a capacity requires a value $\mu(A)$ for each $A \subseteq [m]$. In the case of the Choquet integral, it has therefore been suggested to trade complexity against expressivity by working with $k$-additive measures, which essentially means capturing interactions between criteria up to a degree of $k$ \cite{fallhullermei12,mpub256}. Practically, the full expressivity of a capacity is indeed rarely needed\,---\,on the contrary, from a machine learning point of view it may even be harmful, due to an increased danger of overfitting the training data \cite{fallhullermei12}. Instead, values such as $k=2$ or $k=3$ will typically suffice. The case $k=2$ is especially interesting, and the jump in performance from $k=1$ to $k=2$ is often the highest. This is because, whereas $k=1$ is not able to capture any dependencies, $k=2$ is able to capture pairwise dependencies, and thus dependencies of higher order at least indirectly.

The qualitative analogue of $k$-additivity is $k$-maxitivity \cite{Radko,brabant2018k}. 
Formally, the notions of $k$-maxitive capacity and $k$-maxitive aggregation function are defined as follows:
A capacity $\mu$ is called $k$-maxitive if for any subset $U \subseteq [m]$ with $|U| >k$, there exists a proper subset $V$ of $U$ such that $\mu(V) =\mu(U)$. For $k>1$, $\mu$ is called proper $k$-maxitive if it is $k$-maxitive but not ($k-1$)-maxitive. Note that the $k$-maxitivity of a capacity $\mu$ can be characterized equivalently by the following condition:
\begin{align*}
\mu(U)= \bigvee_{V \subset U} \mu(V) \quad \text{whenever} \quad |U| >k
\end{align*}
A Sugeno integral $S_{\mu}$ is $k$-maxitive, if the underlying measure $\mu$ is $k$-maxitive.
Regarding the issue of complexity, note that a $k$-maxitive capacity requires the specification of 
$$
\sum_{i=1}^k {m \choose i} 
$$
values $\mu(A)$, which, for small to moderate $k$, is substantially less than the $2^m -2$ values needed for the general case, and remains polynomial in $k$.

\section{The Sugeno Integral for Binary Classification}
\label{VC}

A binary classifier is a map $h: \, \mathcal{X} \rightarrow \{0,1 \}$, where $\mathcal{X}$ is a so-called \emph{instance space}; here, we make the common assumption that instances are described in terms of attributes or features, i.e., we assume $\mathcal{X} = \mathcal{X}_1 \times \ldots \times \mathcal{X}_m$, where $\mathcal{X}_i$ is the domain of the $i^{th}$ attribute. Thus, a binary classifier accepts any instance $\vec{x} \in \mathcal{X}$ as an input, and either assigns it to the negative ($h(\vec{x})=0$) or to the positive class ($h(\vec{x})=1$). The learning task essentially consists of choosing an appropriate classifier $h$ from an underlying \emph{hypothesis space} $\mathcal{H}$, given a set of training data\,---\,we will come back to this task in Section~\ref{learning} below.   

We are interested in hypotheses that are expressed in terms of the Sugeno integral. More specifically, we consider a hypothesis space $\mathcal{H}$ consisting of threshold classifiers of the following form, which we simply refer to as ``Sugeno classifiers'': 
\begin{equation}\label{4}
h(\vec{x}) = h(x_1, \ldots, x_m) = \mathbb{I}\Big(S_{\mu}(f(\vec{x})) \geq \beta \Big) \, ,
\end{equation}
where $\mathbb{I}(\cdot)$ is the indicator function and
$$
f(\vec{x}) = \big(f_1(x_1), \ldots , f_m(x_m) \big) = (u_1, \ldots , u_m) \in [0,1]^m \, .
$$
Thus, given an instance $\vec{x} = (x_1, \ldots , x_m)$, a classification is accomplished in three steps:
\begin{itemize}
\item First, each feature $x_i$ is turned into a local utility $u_i$ using the transformation $f_i$.
\item The local utilities, considered as criteria, are combined into an overall utility using the Sugeno integral with capacity $\mu$. 
\item The class assignment is done via thresholding, i.e., by comparing the overall utility with a threshold $\beta$.
\end{itemize}  
Note that a hypothesis $h \in \mathcal{H}$ is thus specified by three components: the transformation $f$, the capacity $\mu$, and the threshold $\beta$. The corresponding hypothesis space $\mathcal{H}$ is relatively rich and allows for modeling classification functions in a very flexible way, especially if the capacity $\mu$ can be chosen without any restrictions. In fact, we can prove the following result about the VC dimension\footnote{The Vapnik-Chervonenkis (VC) dimension is an measure of flexibility of a hypothesis space, which plays an important role in generalization and statistical learning theory \cite{vapn_sl98}.} of $\mathcal{H}$ as defined above (see Appendix \ref{sec:appa}).

\begin{theorem}
The VC dimension of the hypothesis space $\mathcal{H}$ comprised of threshold classifiers of the form (\ref{4}) grows asymptotically at least as fast as $2^m / \sqrt{m}$.
\end{theorem}

It is also interesting to note that the class of Sugeno classifiers covers several types of classifiers, which are commonly used in machine learning, as special cases. This includes, for example, $k$-of-$m$ classifiers, which assign an instance to the positive class if at least $k$ of the set of $m$ criteria are fulfilled, and to the negative class otherwise. More specifically, (\ref{4}) can be specialized to this type of classifier as follows:
\begin{itemize}
\item Features $x_i$ are transformed into binary utilities via $f_i: \mathcal{X}_i \rightarrow \{0,1\}$; thus, $f_i$ simply distinguishes between ``good'' and ``bad'' values $x_i$. 
\item The capacity $\mu$ is defined by $\mu(A) =  |A|/m$. 
\item The threshold $\beta$ takes the value $\nicefrac{k}{m}$. 
\end{itemize}
Another interesting setting is as follows:
\begin{itemize}
\item Again, features $x_i$ are transformed into binary utilities $u_i \in \{ 0,1 \}$, suggesting that a criterion is either satisfied or not. 
\item The capacity $\mu$ is specified by a set of subsets $A_1, \ldots, A_J \subseteq [m]$ as follows: $\mu(A) = 1$ if $A_j \subseteq A$ for some $j \in [J]$, and $\mu(A) = 0$ otherwise. 
\item The threshold $\beta$ takes the value $\nicefrac{1}{2}$.
\end{itemize}
Here, each $A_j = \{ i_1, \ldots , i_j \}$ can be thought of as a rule of the following form:
$$
\text{ IF } (u_{i_1} = 1) \text{ AND } \ldots \text{ AND } (u_{i_j} = 1) \text{ THEN } h=1
$$
The overall classification is positive if at least one of the rules applies, and negative otherwise. Models of this kind are closely related to monotone decision rules \cite{demb_lr09} and monotone decision trees \cite{poth_ct02}.

More generally, a Sugeno classifier with threshold $\beta$ can be interpreted as a logical formula, namely as a disjunctive normal form:
\begin{equation}\label{eq:fkom}
h(\vec{x}) = \bigvee_{\text{boundary\,sets\,}A} \, \bigwedge_{j \in A} \mathbb{I} \big(f_i(x_i) \geq \beta \big) \, ,
\end{equation}
where a boundary set $A$ is a subset $A \subseteq [m]$ such that $\mu(A) \geq \beta$ and $\mu(A') < \beta$ for all $A' \subsetneq A$, i.e., a ``minimal'' subset reaching the threshold $\beta$. Note that, for a given measure $\mu$, the set of boundary sets of $[m]$ forms an antichain, which means that these sets are non-redundant. Also note that, in the case of $k$-maxitive measures, the size of boundary sets is at most $k$. Considering the threshold $\beta$ as a kind of aspiration level, the condition $\mathbb{I}(f_i(x_i) \geq \beta)$ in  (\ref{eq:fkom}) can be interpreted as ``feature $x_i$ is satisfactory''. Correspondingly, the Sugeno classifier assigns the positive class if $\vec{x}$ is satisfactory on all features in at least one boundary set of features $A$.

\begin{proposition}
The Sugeno classifier (\ref{4}) coincides with (\ref{eq:fkom}).
\end{proposition}
\begin{proof}
Suppose that $h(\vec{x}) = 1$ according to (\ref{eq:fkom}). Thus, there exists a boundary set $A$ such that $\mathbb{I}(f_j(x_j) \geq \beta )$, i.e., $u_j \geq \beta$ for all $j \in A$. Now, consider the representation (\ref{2}) of the Sugeno integral. Since $\min_{j \in A} u_j \geq \beta$ and $\mu(A) \geq \beta$, we have $S_\mu(\vec{x}) \geq \beta$ and $h(\vec{x}) = 1$ according to (\ref{4}). 

Suppose that $h(\vec{x}) = 1$ according to (\ref{4}). Thus, according to (\ref{2}), there exists at least one subset $A' \subseteq [m]$ such that $u_j \geq \beta$ for all $j \in A'$ and $\mu(A') \geq \beta$. Therefore, $A'$ is either a boundary set or a superset of a boundary set, so that there exists a boundary set $A \subseteq A'$. Moreover, since $\min_{j \in A} u_j \geq \min_{j \in A'} u_j \geq \beta$, we conclude that $h(\vec{x}) = 1$ according to (\ref{eq:fkom}).
\end{proof}

\section{Learning a Sugeno Classifier}
\label{learning}

Consider the class of Sugeno classifiers, i.e., the hypothesis space $\mathcal{H}$, as introduced in the previous section. More specifically, following the idea of structural risk minimization, we structure the space as follows: $\mathcal{H}_1 \subset \mathcal{H}_2 \subset \cdots \subset \mathcal{H}_m$, where $\mathcal{H}_k$ denotes the class of Sugeno classifiers restricted to $k$-maxitive capacities. The parameter $k$ will serve as a hyper-parameter, i.e., as a parameter of the learning algorithm.

Given a set of training data of the form  
\begin{equation}\label{E1}
\mathcal{D}= \left\{\left( \vec{x}^{(i)}, y^{(i)}\right)  \right\}_{i=1}^{n} \subset \left(\mathcal{X} \times \{ 0,1 \}  \right)^n \, ,
\end{equation}
supposed to be generated i.i.d.\ (independent and identically distributed) according to an underlying (though unknown) probability measure $\textbf{P}$ on $\mathcal{X} \times \{ 0,1 \}$,
the task of the learning algorithm (or learner for short) is to find a Sugeno classifier $h \in \mathcal{H}_k$ with low risk (expected loss) 
\begin{equation}\label{E2}
R(h)= \int_{\mathcal{X} \times \{ 0,1 \} } \ell \left(  y , h(\vec{x}) \right) \,  d \, \textbf{P}(\vec{x}, y) \, ,
\end{equation}
where $\ell$ is a loss function such as $0/1$ loss ($\ell(y, \hat{y}) = 0$ if $y = \hat{y}$ and $=1$ otherwise). Recall that a classifier is identified by the feature transformation $f$, the capacity $\mu$, and the threshold $\beta$, which all have to be determined on the basis of the data $\mathcal{D}$. In the following, we discuss these components one by one.

\subsection{Feature Transformation}

As already explained, feature transformation is meant to turn each feature (predictor variable)  into a criterion that is measured in the unit interval, that is, to replace each feature value $x_i$ by a local utility score $u_i$. More specifically, feature transformation is assumed to assure monotonicity in the sense that higher values $u_i$ are better (more likely to produce the positive class), as well as commensurability between the criteria. 

Here, we make the assumption that monotonicity already holds for the original features, so that the transformations $f_i$ can be monotone as well. More precisely, we assume that higher values $x_i$ are better. If the opposite is the case (lower values are better), all values $x_i$ can simply be replaced by $-x_i$. Sometimes, the direction might not be be known beforehand, and instead must be determined on the basis of the data. In Appendix \ref{sec:appb}, we propose a method for that purpose.

In our current approach, the transformations $f_i$ are determined in an unsupervised manner, i.e., using the feature information $\vec{x}^{(1)}, \ldots , \vec{x}^{(n)}$ in the training data $\mathcal{D}$ but not the observed class labels $y^{(1)}, \ldots , y^{(n)}$. 
Starting from ``the higher the better'' feature values, transformation essentially comes down to normalizing the data, i.e., mapping the feature values (isotonically) to the unit interval. To this end, we make use of a quantile-based approach, that is, the idea to replace a value $x$ by the probability $P(X \leq x)$, where $P$ is the underlying distribution of the feature. Since $P$ is not known, it is replaced by the empirical distribution function. 

More precisely, the transformation $f_i$ of the $i^{th}$ feature is determined as follows. For a suitable permutation $\sigma$ on $[n]$, denote by $x_{i,\sigma(1)} \leq x_{i,\sigma(2)} \leq \cdots \leq x_{i,\sigma(n)}$ the sorted list of values that are observed for the $i^{th}$ feature in the training data $\mathcal{D}$. We then define the empirical distribution function $f_i: \mathbb{R} \rightarrow [0,1]$ in terms of a piecewise linear interpolation of the points $(x_{i,\sigma(j)}, a_j)$, where
\begin{equation*}
a_j = \frac{1}{2} \Big( 
 \big\vert  \{ l \, | \, x_{i,\sigma(l)} < x_{i,\sigma(j)} \} \big\vert 
+   \big\vert  \{ l \, | \, x_{i,\sigma(l)} \leq x_{i,\sigma(j)} \} \big\vert 
\Big) 
\end{equation*}
Note that the number of data points to be interpolated is not necessarily $n$, because tied $x$-values $x_{i,\sigma(j)} = x_{i,\sigma(j+1)}$ only contribute a single point. Averaging the cases of strict inequality $<$ and $\leq$ in the computation of $a_j$ is important for exactly those ties, which occur especially often on ordinal scales. To make the definition of $f_i$ complete, we set $f_i(x)= 0$ for $x < x_{i,\sigma(1)}$ and $f_i(x)=1$ for  $x > x_{i,\sigma(n)}$. 

\subsection{Learning the Capacity}

To learn the capacity $\mu$, we follow the principle of empirical risk minimization. Thus, we consider the problem of minimizing the 0/1 loss of the classifier \eqref{4} on the training data \eqref{E1}. This problem is provably NP-hard, so we cannot expect to find an optimal solution efficiently. Therefore, we opt for an approximate solution. To this end, we construct a linear program LP, the solution of which will determine the capacity $\mu$. In spite of the theoretical (worst-case) complexity of linear programming, this approach leads to a practically efficient learning algorithm, thanks to the availability of modern solvers that are able to handle programs with thousands of  variables and inequalities within seconds.

The inequalities of LP are coming from the monotonicity of the capacity $\mu$. For $A \subseteq [m]$, let $c_A$ denote the value $\mu(A)$, i.e., the value assigned by the capacity $\mu$ to $A$. Thus, the set $C = \{c_A \, | \, A \subset [m]\}$ corresponds to the parameters that need to be determined. Since $c_{\emptyset} = 0$ and $c_{[m]} = 1$, the following inequalities need to be added to the optimization problem:
\begin{equation}\label{E5}
\forall A \subseteq [m], b \in [m] \slash A : c_A \leq c_{A \cup \{b\}} 
\end{equation}
When fitting a $k$-maxitive capacity, some of these inequalities may actually turn into equalities. 

Let $\left( \vec{x}^{(i)}, y^{(i)}\right)$ be an instance in the training data \eqref{E1}. For simplicity, we subsequently drop the superscript and simply write $\vec{x}=\left(x_1, \ldots, x_m \right)$. Let $\sigma$ be a permutation such that $x_{\sigma(1)} \leq x_{\sigma(2)} \leq \ldots \leq x_{\sigma(m)}$. Using the median-representation of the Sugeno integral, we have
\begin{equation*}
S_{\mu}(\vec{x})= \text{Med}\Big(x_1, \ldots, x_m, \mu\left(A_{\sigma(2)}\right), \ldots, \mu\left(A_{\sigma(m)}\right)\Big),
\end{equation*}
where $A_{\sigma(i)}:=\{\sigma(i), \ldots, \sigma(m)\}$. 

Suppose that $\vec{x}$ is a positive example, i.e., $y = 1$. According to the classifier \eqref{4} for a given $\beta$, we must guarantee that $S_{\mu}(\vec{x}) \geq \beta$, which is equivalent to guaranteeing that at least $m$ among the values $x_1, \ldots, x_m, c_{A_{2}}, \ldots, c_{A_{m}} \geq \beta$. Inspired by the of margin maximization \cite{scho_lw}, we would even like to guarantee $S_{\mu}(\vec{x}) \geq \beta^+  = \beta + \rho$, where $\rho \geq 0$ is a margin. Let $p= \min\{j \, | \, x_j \geq \beta^+\}$. If $p= m$, then the above condition is automatically fulfilled, and $S_{\mu}(\vec{x}) \geq \beta^+$. In this case, no constraint on the measure $\mu$ needs to be added. Suppose that $p<m$. In this case, to satisfy the condition $S_{\mu} \geq \beta^+$, at least $m-p$ of the values $c_{A_{2}}, \ldots, c_{A_{m}}$ must be $\geq \beta^+$. In light of the order relations $c_{A_{2}} \geq \cdots \geq c_{A_{m}}$ and $c_{A_{1}} = \cdots = c_{A_{m-k}}$, which guarantees the $k$-maxitivity of $\mu$, we set $p=m-k+2$ for all $1< p < m-k+2$. Thus, based on the value of $p$, the constraint $c_{A_{p-1}} \geq \beta^+$ is added. We ignore the particular case of $p=1$.

Now, suppose that $\vec{x}$ is a negative example, i.e., $y = 0$. To classify $\vec{x}$ correctly (with a margin $\rho$), we must guarantee that $S_{\mu} < \beta^- = \beta - \rho$, which is equivalent to guaranteeing that at least $m$ among the values $x_1, \ldots, x_m, c_{A_{2}}, \ldots, c_{A_{m}} < \beta^-$. Let $p= \max \{j \, | \, x_j < \beta^-\}$. If $p= m$, then the above condition is automatically fulfilled, and $S_{\mu}(\vec{x}) < \beta^-$. In this case, no constraint on the measure $\mu$ needs to be added. Let $p<m$. In this case, to satisfy the condition $S_{\mu} < \beta^-$, at least $m-p$ of the values $c_{A_{2}}, \ldots, c_{A_{m}}$ must be $< \beta^-$. In light of the order relations $c_{A_{2}} \geq \ldots \geq c_{A_{m}}$, 
we set $p=m-k$ for all $p < m-k$. Therefore, based on the value of $p$, the constraint $c_{A_{p+1}} < \beta^-$ is added.

For each example $(\vec{x}^{(i)},y^{(i)}) \in \mathcal{D}$, a constraint can be derived according to the procedure outlined above. Obviously, it will not always be possible to satisfy all these constraints simultaneously, i.e., to fit the training data without any error. Therefore, to account for unavoidable mistakes, we introduce slack variables $\xi_i$. When $\vec{x}^{(i)}$ is a positive example, the corresponding inequality becomes $c_{A_{p_i- 1}}+ \xi_i \geq \beta^+$. Likewise, when $\vec{x}^{(i)}$ is a negative sample, the relaxed condition is $c_{A_{p_i + 1}}- \xi_i < \beta^-$. Noting that $2y^{(i)}-1 = +1$ for positive and $=-1$ for negative examples, both conditions can be expressed as follows:
\begin{equation}\label{eq:slc}
- (2y^{(i)}-1) (c_{A_{p_i -y_{i}}} - \beta) - \xi_i + \rho \leq 0
\end{equation}
Finally, the LP consists of minimizing the sum of slack variables $\xi_i$ subject to the above constraints, that is:
\begin{equation}\label{7}
\operatorname*{minimize}_{C , \xi_1, \ldots , \xi_n} \; \sum_{i=1}^n  \xi_i \quad \text{subject to}\quad (\ref{E5}) \text{ and } (\ref{eq:slc})
\end{equation}
A slight modification of the above program is required for the case of learning a $k$-maxitive capacity. In this case, the variables $c_A$ for $|A| > k$ are implicitly defined by the condition $c_A = \max\{ c_B \, | \, B \subset A \}$, and hence not part of the program. In the case of a positive example, to ensure  $S_{\mu}(\vec{x}) \geq \beta^+$, we need to guarantee that $c_A \geq \beta^+$ for at least one subset $A \subseteq B= \{x_p, \ldots , x_m\}$. If $|B| \leq k$, we can simply add the constraint $c_B \geq \beta^+$. Otherwise, we need to assure that $c_A \geq \beta^+$ for at least one proper $k$-subset $A \subset B$. This ``existential'' constraint is of disjunctive nature, and therefore difficult to formalize in terms of an LP. To deal with this problem, we add a single constraint $c_A \geq \beta^+$ for a randomly chosen $k$-subset of $B$. Obviously, this condition is sufficient though not necessary, i.e., our program might be slightly more constrained than necessary. This, however, appears acceptable in light of our idea of using $k$-maxitivity for the purpose of regularization. The case of a negative example can be handled more easily. Here, we need to satisfy $c_A \leq \beta^-$ for all $k$-subsets $A \subseteq B= \{x_p, \ldots , x_m\}$.

Just like the order of maxitivity $k$, the margin $\rho$ is a hyper-parameter of the algorithm that needs to be fixed, for example, through an internal validation procedure. As one advantage in comparison to other margin classifiers, let us mention that our Sugeno classifier produces normalized predictions in the unit interval. Obviously, this simplifies the search for ``reasonable'' values of the margin, which we should expect to be found close to 0, for example in the range $[0,0.1]$ or $[0,0.2]$.

\subsection{Learning the Threshold}

The linear programming approach to learning the capacity $\mu$, as outlined in the previous section, assumes the threshold $\beta$ to be given\,---\,by treating both $\mu$ and $\beta$ as variables, the program would no longer be linear. Correspondingly, the threshold needs to be determined first, prior to solving (\ref{7}).

To this end, we again formulate a linear program, which seeks to find a $\beta$ that is optimal in the sense of minimizing the number of classification errors on the training data $\mathcal{D}$. Yet, since the capacity $\mu$ is not (yet) known, we have to replace the Sugeno classifier by a ``surrogate'' classifier. Here, we suggest to classify an instance $\vec{x}$ as positive if
$$
\text{Med}(u_1, \ldots , u_m) = \text{Med} \big(f_1(x_1),\ldots, f_m(x_m) \big) \geq \beta \, ,
$$
and as negative otherwise. This classifier follows the same principle as the Sugeno classifier, namely the median rule, though without using the values of the capacity $\mu$. The basic idea is to find a threshold that is well attuned to the feature values, so that the capacity can be used later on to optimally ``refine'' or correct the classifications.

In summary, we end up with the following linear program, in which the slack variables $\zeta_i$ are again used to account for those mistakes that cannot be avoided:
\begin{align}\label{10}
& \operatorname*{minimize}_{\beta , \zeta_1, \ldots , \zeta_n} \; \sum_{i=1}^n \zeta_i \\
&  \text{subject to}\quad   \nonumber \\
& - (2y^{(i)}-1) \Big(\text{Med} \left(f_1(x_{1}^{(i)}),\ldots, f_m(x_{m}^{(i)}) \right) - \beta \Big) - \zeta_i \leq 0 \nonumber
\end{align}
for $i=1, \ldots , n $.

\section{\textbf{Complexity reduction}}
\label{Complexity}

As discussed before, restricting the Sugeno integral to $k$-maxitive capacities and estimating the measure $\mu$ on subsets of size at most $k$ may be advantageous from a learning point of view, as it reduces the danger of poor generalization due to over-fitting the training data.  
In this section, we elaborate on how to efficiently determine the hyper-parameter $k$ in the most favorable way.
A practical and quite obvious approach is to determine an optimal $k$ in a data-driven way through cross-validation (on the training data), i.e., trying different values for $k$ and adopting the one that leads to the best estimated generalization performance.
However, from a theoretical point of view, one may wonder what ensures the existence of such $k$, and how significant it is to push the $k$-maxitivity property to an arbitrary capacity $\mu$. 

Once an SI model has been fit, $\mu$ can be considered as an approximately $k$-maxitive capacity, for some suitable $k \in [m]$. 
In mathematical analysis, the significance of this sort of approximation, i.e., the question of how an approximate object can be estimated by an exact object, is studied under the notion of \emph{stability}
\cite{gordji2016theory,hyers1941stability}.

When we constrain the Sugeno integral by $k$-maxitivity, we can still have a significant approximation of that Sugeno integral according to the following stability result. For an arbitrary measure $\mu$ on $2^{[m]}$ and a small real number $\varepsilon \in [0,1),$ we define the following subset of $[m]$:
\begin{align*}
G_{\mu,\varepsilon} = \left\{ k \in [m] ~|~ \mu(B)- \underset{A \subseteq B, |A| \leq k}{\bigvee} \mu(A) \leq \varepsilon, \forall B \in 2^{[m]}\setminus\emptyset \right\}.
\end{align*}

\begin{theorem}\label{T51}
Let $\mu$ be a capacity on $[m]$. For a given $\varepsilon\in [0,1),$ if the set $G_{\mu,\varepsilon}$ attains the minimum at $k^*$, there exists a $k^*$-maxitive measure $\mu^*$ such that
\begin{equation}\label{51}
0 \leq S_{\mu}(\vec{x})- S_{\mu^*}(\vec{x}) \leq \varepsilon
\end{equation}
for all $(\vec{x}, y) \in \mathcal{D}$.
\end{theorem}

\begin{proof}
For each $\varepsilon \in (0,1),$ the set $G_{\mu,\varepsilon}$ is non-empty, because $m \in G_{\mu,\varepsilon}$ for any $\varepsilon \in (0,1)$. It means that $G_{\mu,\varepsilon}$ is a non-empty finite subset of integer numbers, so it takes its minimum for some $k \in [m]$, say $k^*$. Corresponding to $k^*$, we define the measure $\mu^*$ for each $B \in 2^{[m]}\setminus\emptyset$ as follows:
$$\mu^*(B)= \begin{cases} \mu(B) & |B| \leq k^*, \\
									      \underset{A \subseteq B, |A|= k^*}{\bigvee} \mu(A) & \text{otherwise}.
\end{cases}$$
One can easily check $k^*$-maxitivity properties (cf. Sect. \ref{Sugeno}) for $\mu^*$. 

Assume that $(\vec{x}, y)$ is an instance in the training set $\mathcal{D}$ and $\sigma$ is a permutation of $[m]$ such that $x_{\sigma(1)} \leq x_{\sigma(2)} \leq \ldots \leq x_{\sigma(m)}$. From $k^*$-maxitivity of $\mu^*$, it follows that
\begin{equation}\label{8}
\begin{split}
\mu^{*}\left(A_{\sigma(m)}\right) & \leq \mu^{*}\left(A_{\sigma(m-1)}\right)  \leq  \cdots \leq \mu^{*}\left(A_{\sigma(m-(k^*-1))}\right) = \\
& = \mu^{*}\left(A_{\sigma(m-k^*)}\right) =\cdots=\mu^{*}\left(A_{\sigma(1)}\right),
\end{split}
\end{equation}
where $A_{\sigma(i)}:=\{\sigma(i), \ldots, \sigma(m)\}$. It should be noted that $S_{\mu}(\vec{x})$ is either equal to $x_{\sigma(p)}$ or to $\mu\left(A_{\sigma(p)}\right)$ for some $p \in [m]$ (see \cite[Prop.\ 5.65]{grabisch2009aggregation}).
Indeed, $p$ corresponds to the place where the values $x_{\sigma(j)}$ cross the values $\mu\left(A_{\sigma(j)}\right)$. 

If $p \geq m-(k^*-1)$, then by the definition of $\mu^{*}$, we have $\mu\left(A_{\sigma(p)}\right)= \mu^{*}\left(A_{\sigma(p)}\right)$, and consequently $S_{\mu}(\vec{x})=S_{\mu^{*}}(\vec{x})$. Now, assume that $p \leq m-k^*$. We can distinguish the following two cases:
\begin{itemize}
\item[(1)] If $x_{\sigma(p)} \leq \mu\left(A_{\sigma(p)}\right)$, then  
\begin{equation}\label{8}
S_{\mu}(\vec{x})= x_{\sigma(p)}
\end{equation}
according to the median-description of the Sugeno integral.
Since $\mu^*$ is $k^*$-maxitive and $k^* \in G_{\mu,\varepsilon}$, it is not difficult to deduce that
\begin{equation}\label{9}
\mu\left(A_{\sigma(p)}\right) \leq \mu^{*}\left(A_{\sigma(p)}\right)+ \varepsilon \, ,
\end{equation} 
and therefore, $x_{\sigma(p)} \leq \mu^{*}\left(A_{\sigma(p)}\right)+ \varepsilon$. If $x_{\sigma(p)} \leq \mu^{*}\left(A_{\sigma(p)}\right)$, then $S_{\mu^{*}}(\vec{x})=x_{\sigma(p)}=S_{\mu}(\vec{x})$, otherwise $x_{\sigma(p)} > \mu^{*}\left(A_{\sigma(p)}\right)$. Thus, it follows that $S_{\mu^{*}}(\vec{x})=\mu^{*}\left(A_{\sigma(p)}\right)$. In this case, from \eqref{8} and \eqref{9}, we conclude that $S_{\mu}(\vec{x})- S_{\mu^*}(\vec{x}) \leq \varepsilon$.

\item[(2)] If $\mu\left(A_{\sigma(p)}\right) \leq x_{\sigma(p)}$, then we have $S_{\mu}(\vec{x})= \mu\left(A_{\sigma(p)}\right)$. From the definition, $\mu^{*}\left(A_{\sigma(p)}\right) \leq \mu\left(A_{\sigma(p)}\right) \leq x_{\sigma(p)}$, and so $S_{\mu^*}(\vec{x})= \mu^{*}\left(A_{\sigma(p)}\right)$. Using again the inequality \eqref{9}, we conclude that $S_{\mu}(\vec{x})- S_{\mu^*}(\vec{x}) \leq \varepsilon$.
\end{itemize}
This concludes the proof. \end{proof}
According to Theorem \ref{T51}, any Sugeno integral $S_{\mu}(\vec{x})$ can be approximated by a $k^*$-maxitive Sugeno integral $S_{\mu^*}(\vec{x})$, where $k^*$ can be chosen optimally from $[m]$.

\section{Experimental Results}
\label{Experimental}

\begin{table*}
\centering
 \caption{Summary of data sets.} \label{tab1}
\scalebox{0.95}{
 \begin{tabular}{lccc}
\hline\noalign{\smallskip}
 data set & \# instances & \# attributes & source \\
\noalign{\smallskip}\hline\noalign{\smallskip}\noalign{\smallskip}\noalign{\smallskip}
Dagstuhl-15512 ArgQuality Corpus (DGS)  & 960   & 14  & Wachsmuth et al. \cite{wachsmuth:2017a}\\
Den Bosch (DBS)  & 120   & 8  & Daniels and Kamp \cite{daniels1999application}\\
Mammographic (MMG)  & 961  & 5  & UCI\\
Auto MPG & 392  & 7  & UCI\\
Employee Rejection/Acceptance (ERA)  & 1000   & 4  & UCI\\
Employee Selection (ESL) & 488   & 4  & WEKA\\
Breast Cancer (BCC)  & 286   & 7  & UCI\\
Breast Cancer Wisconsin (BCW)  & 699   & 9  & UCI\\
Lecturers Evaluation (LEV) & 1000   & 4  & WEKA\\
Haberman's Survival Data (HAB)  & 306  & 3 & UCI\\
Indian Liver Patient (ILP) & 583 & 9 & UCI\\
\noalign{\smallskip}\hline
\end{tabular}}
\end{table*}

In this section, we present results of some experimental studies that we conducted to assess the practical performance of the Sugeno classifier. To this end, we collected a set of suitable benchmark data sets, mostly from the UCI\footnote{\url{http://archive.ics.uci.edu/ml/}} and the WEKA machine learning repositories \cite{hall2009weka}. In particular, these are data sets for which a monotonicity assumption is plausible, and which have already been used in previous studies on monotone classification \cite{fallhullermei12}. An overview of the data sets is given in Table \ref{tab1}; for a detailed description of the data, we refer to Section \ref{appdata} in the appendix.

\subsection{Sensitivity Analysis for Threshold Learning}

As we explained in Section~\ref{learning}, the capacity $\mu$ and the threshold $\beta$ of the Sugeno classifier cannot be learned simultaneously (at least not efficiently). Instead, we determine $\beta$ first, using the linear program \eqref{10}, and fit the capacity $\mu$ afterwards. Therefore, as a natural question, one may ask what is lost by decomposing the learning task into two parts, instead of optimizing $\mu$ and $\beta$ jointly. To study this question, we conducted a kind of sensitivity analysis. 


More specifically, instead of only learning the Sugeno classifier with the threshold determined according to \eqref{10}, we learned a classifier (capacity) for all thresholds in the unit interval (i.e., for $0.01, 0.02, \ldots, 1$) and compared the performances of the classifiers thus obtained. To this end, we estimated the generalization error of each classifier in terms of its error on hold-out data, averaged over several random splits of the data for training and testing.    


Figure \ref{fig3} shows results for the LEV data set. As can be seen, the threshold determined according to (\ref{7}) always leads to optimal overall performance. From these results, which look very similar for other data sets, we conclude that our approach to learning the threshold is well suited, with little scope for further improvement due to more sophisticated methods. 

\begin{figure*}[!htb]
\begin{center}
\subfloat[]{\includegraphics[height=4cm,width=5cm]{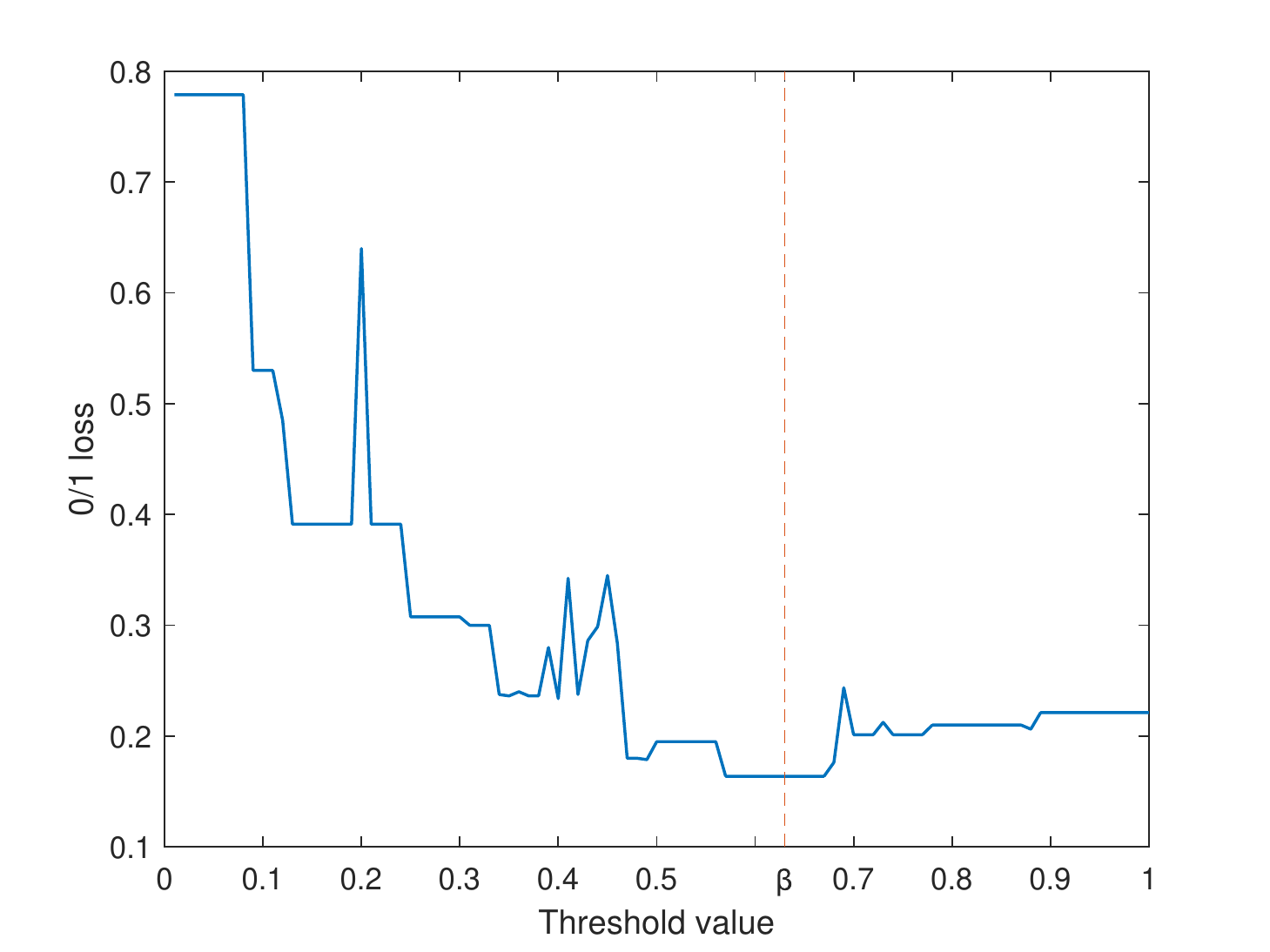}}
\subfloat[]{\includegraphics[height=4cm,width=5cm]{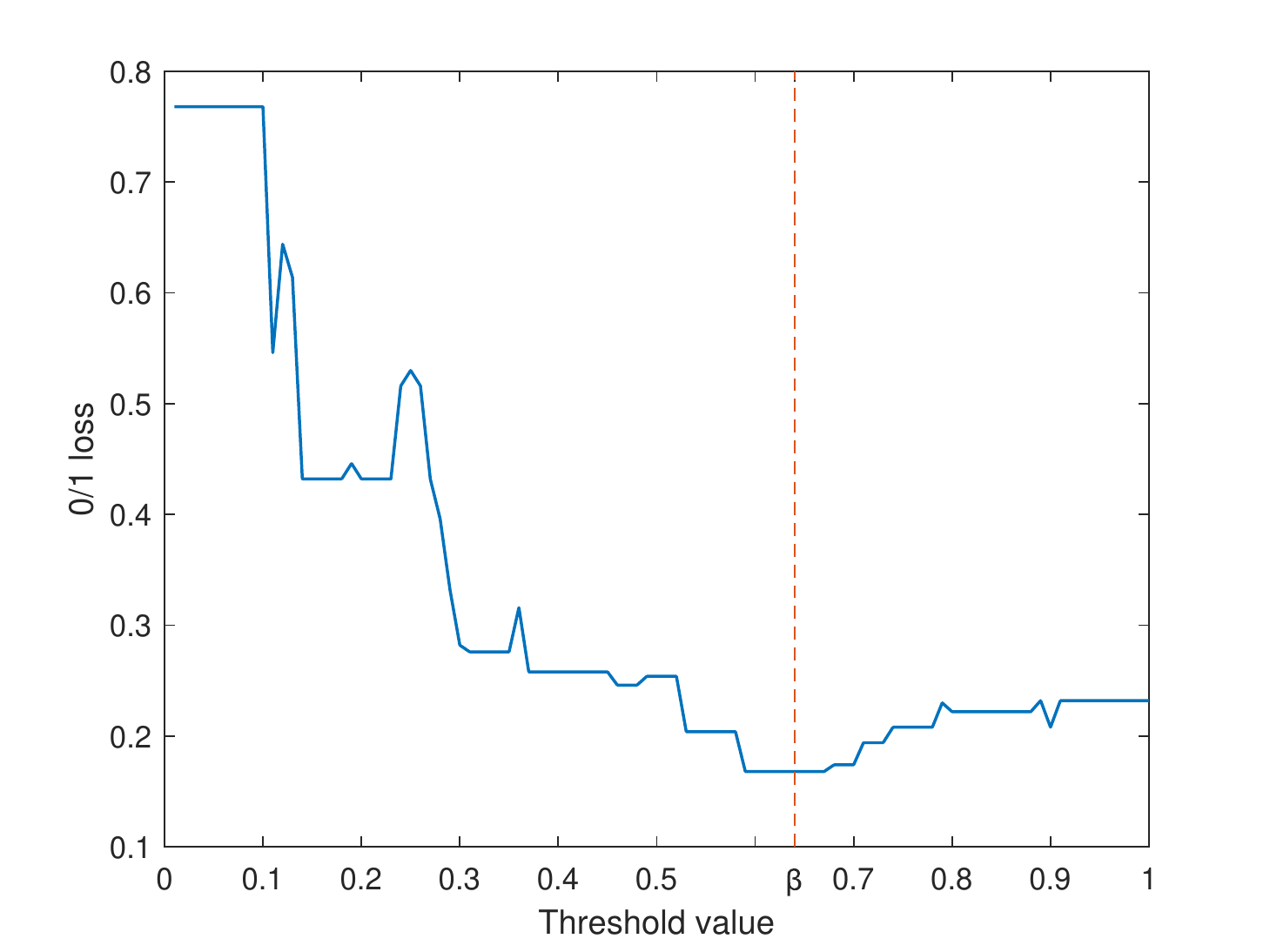}}
\subfloat[]{\includegraphics[height=4cm,width=5cm]{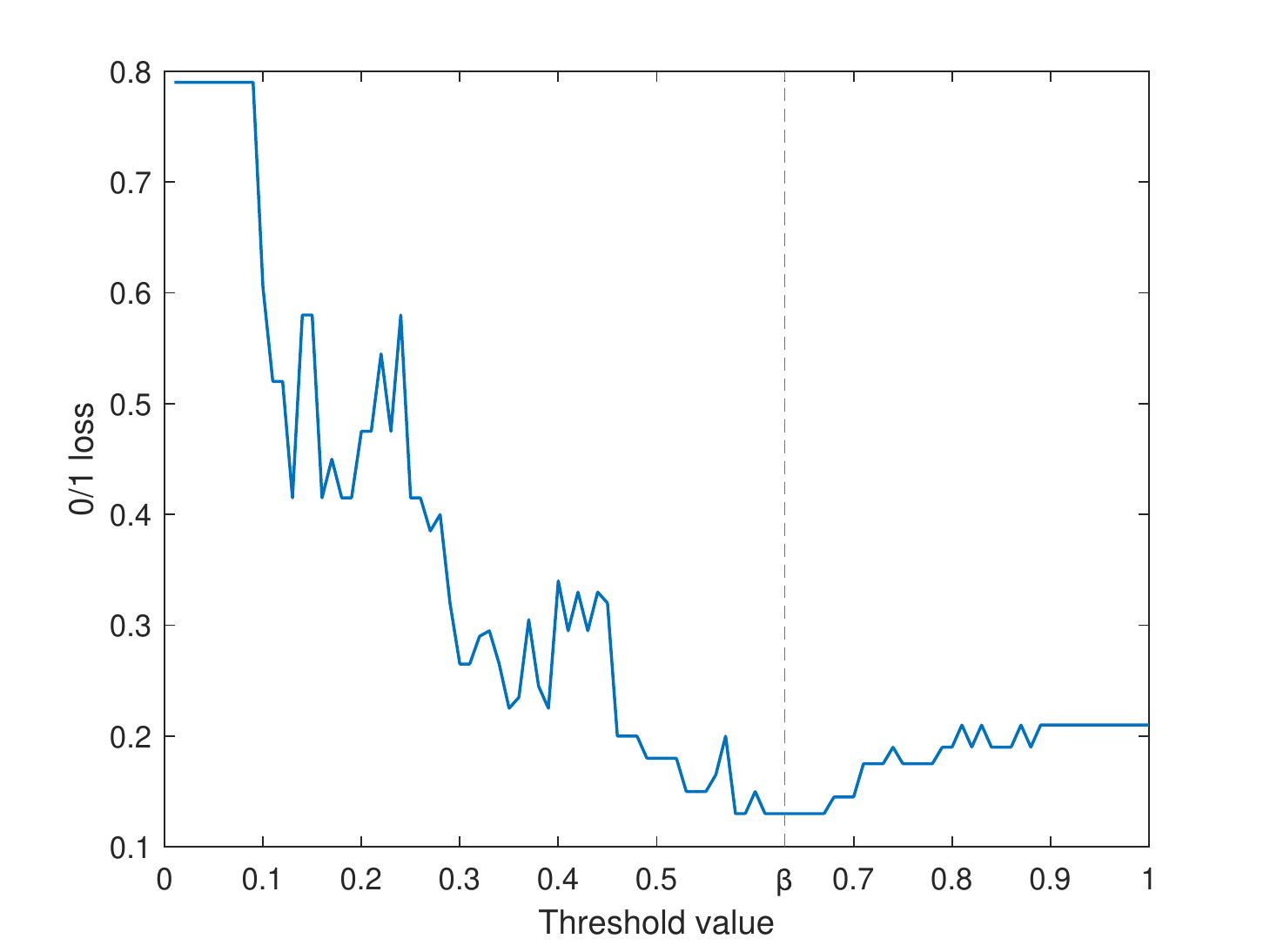}}
\caption{Performance (average $0/1$ loss on test data) of the Sugeno classifier on the LEV data, depending on the choice of the threshold. From left to right: $20\%$, $50\%$, and $80\%$ of the data used for training. The threshold determined by (\ref{7}) is depicted by the vertical line.}\label{fig3}
\end{center}
\end{figure*}

As additional evidence, we ``alternately'' optimized the threshold $\beta$ as follows: Once the capacity $\mu$ (i.e., the set of parameters $\{c_A \, | \, A \subset [m]\}$) has been found based on the original threshold $\beta$ determined by (\ref{7}), we again optimized the threshold given the capacity $\mu$. As before, this can be done by solving an LP. Figure \ref{fig33} compares the threshold $\beta$ determined by (\ref{7}) and the threshold $\beta_0$ determined in this way. As can be seen, both values are very close to each other. 


\begin{figure*}[!htb]
\begin{center}
\subfloat[]{\includegraphics[height=4cm,width=5cm]{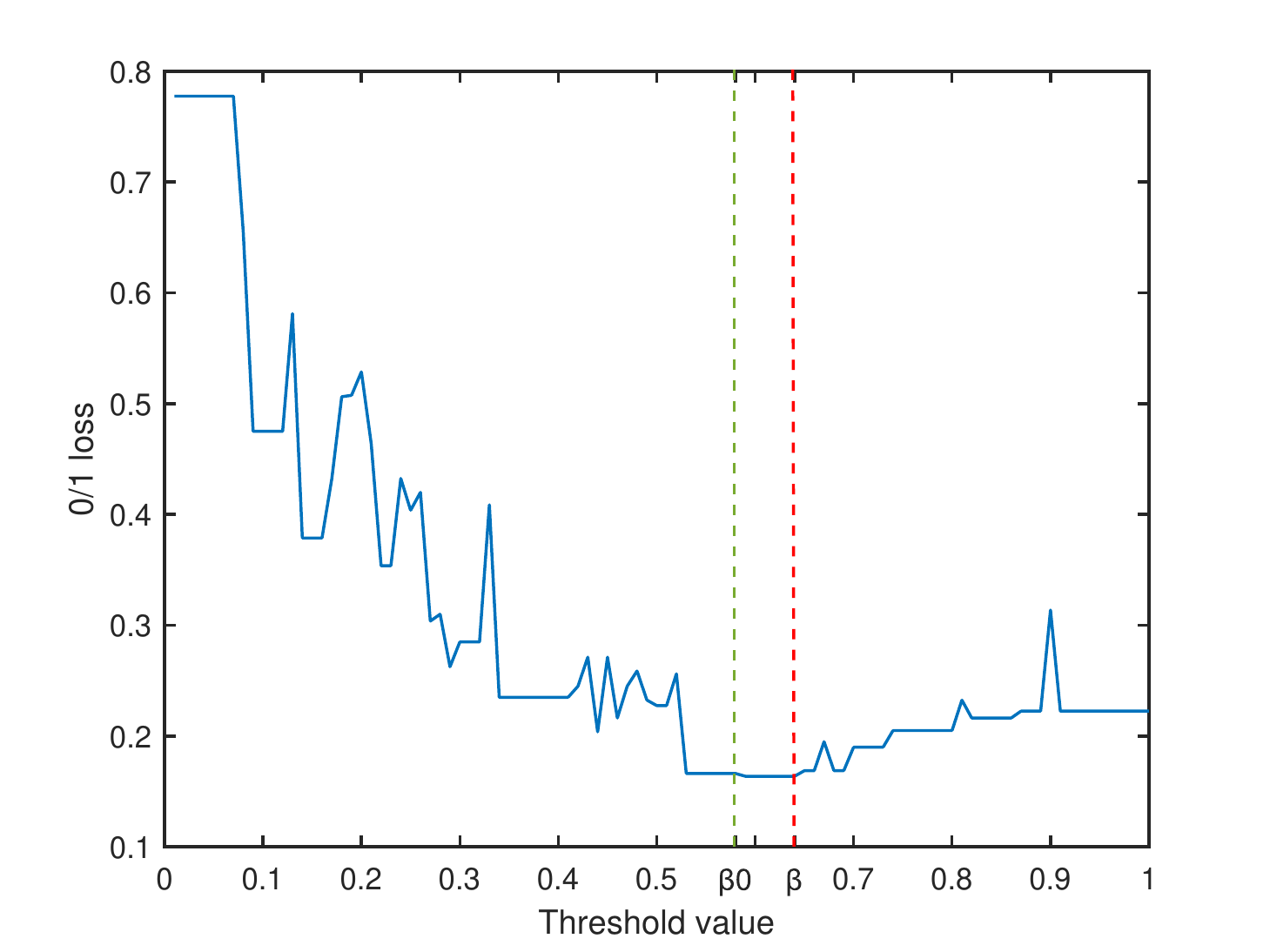}}
\subfloat[]{\includegraphics[height=4cm,width=5cm]{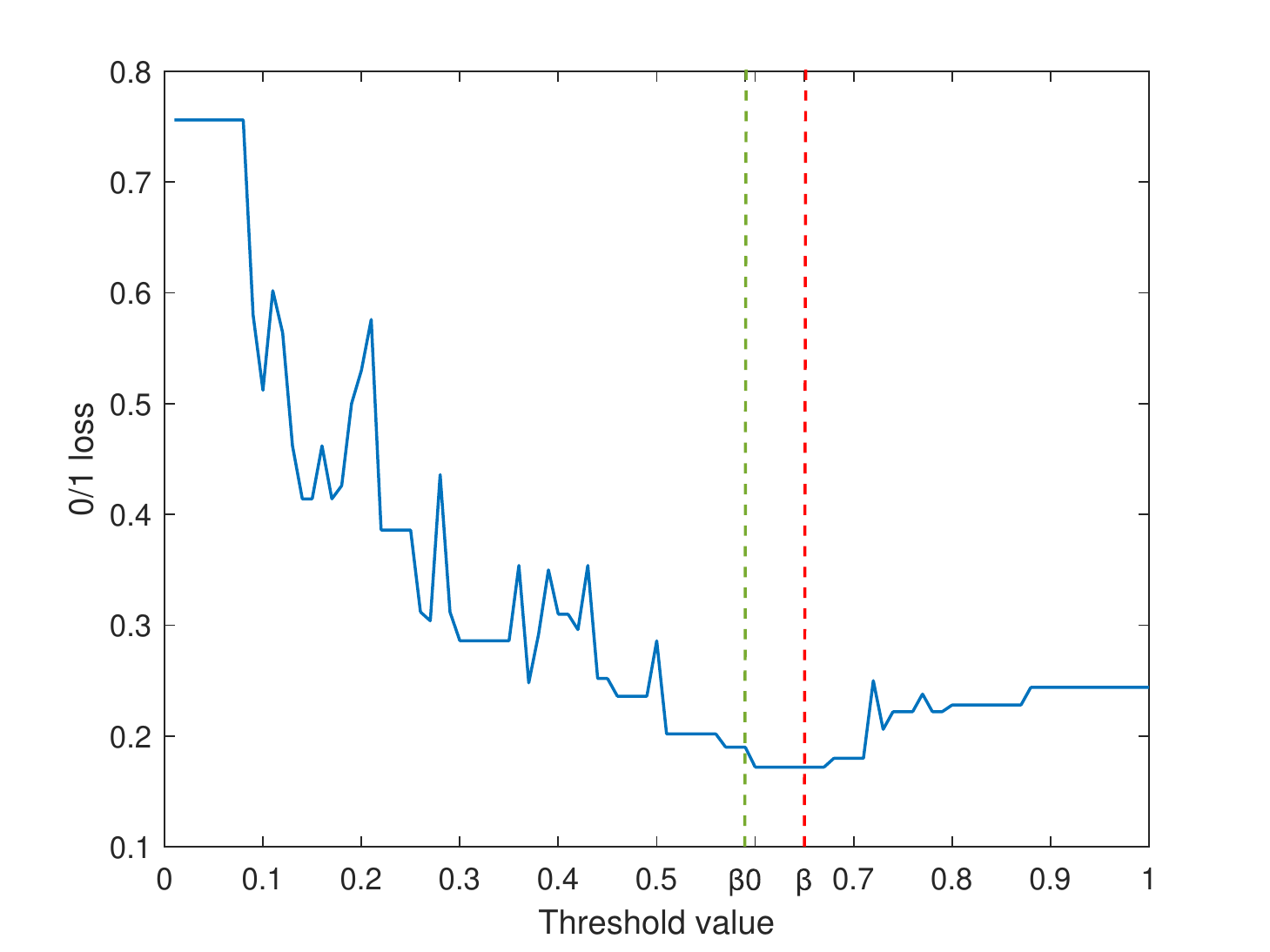}}
\subfloat[]{\includegraphics[height=4cm,width=5cm]{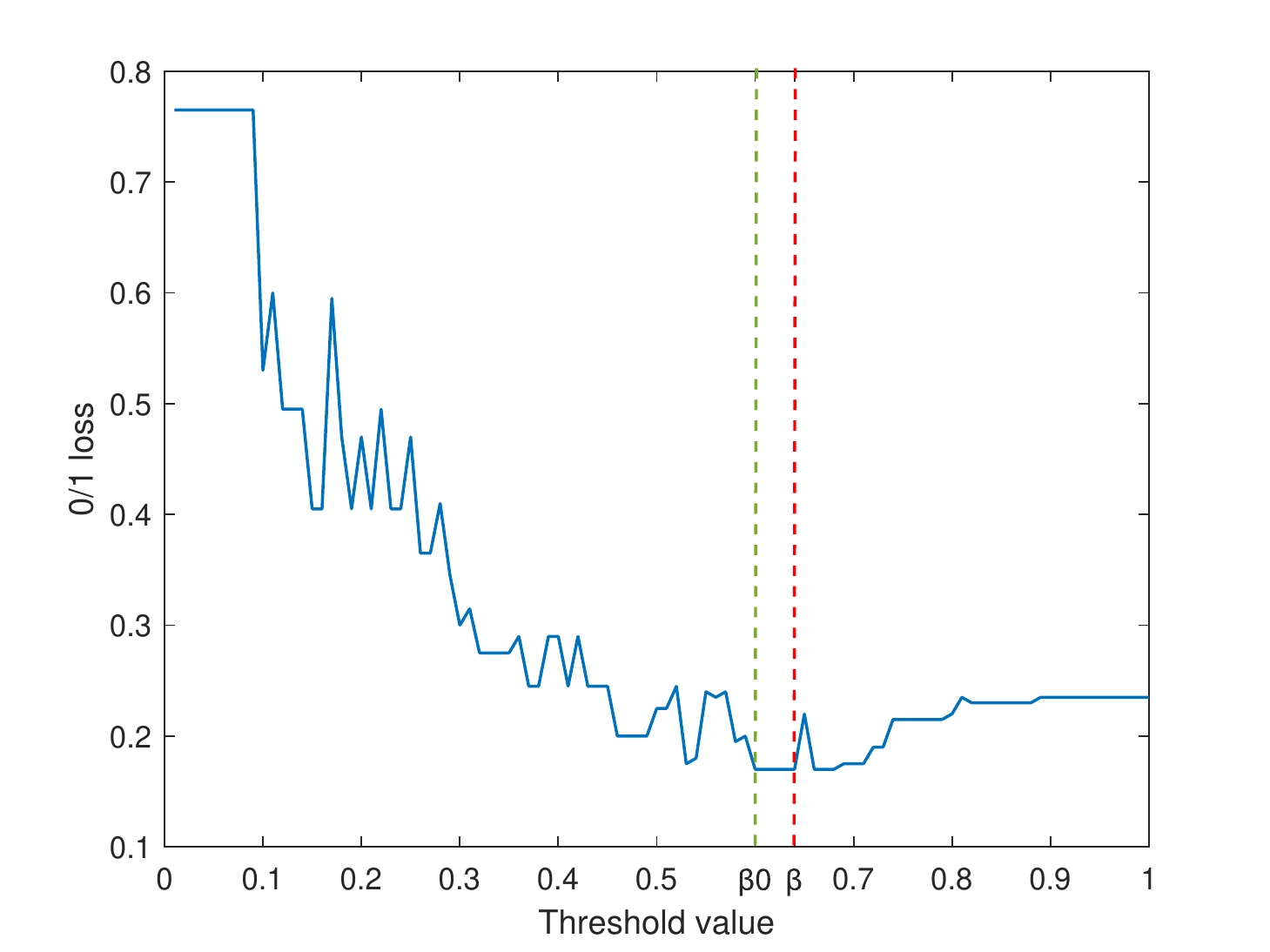}}
\caption{Performance (average $0/1$ loss on test data) of the Sugeno classifier on the MMG data, depending on the choice of the threshold. From left to right: $20\%$, $50\%$, and $80\%$ of the data used for training. The red vertical line depicts the threshold $\beta$ determined by (\ref{7}), and the green vertical line depicts the threshold $\beta_0$ determined by alternate optimization.}\label{fig33}
\end{center}
\end{figure*}

\subsection{Classification Performance}

In this section, we investigate the performance of the Sugeno classifier in terms of its classification accuracy. In all experiments, the data is randomly split into two parts, one for training and one for testing. Performance is then reported in terms of the $0/1$ loss on the test data, averaged over 100 random splits. The Sugeno classifier is determined on the training data as described in Section~\ref{learning}. The degree $k$ of ``maxitivity'' is a hyper-parameter of the learning method, which is determined in an internal cross-validation: To estimate the performance that can be achieved with a certain $k$, the learner conducts a 10-fold cross-validation on the training data. Trying different values for $k$, it picks the presumably best one for training the classifier to be used for prediction on the test data. 

As a baseline to compare with, we used a $2$-additive ``Choquistic regression'' (CR) as proposed in \cite{fallhullermei12}. As already explained in the beginning of the paper, this approach is very close to ours and essentially differs in using the Choquet instead of the Sugeno integral as an aggregation function. As additional baselines, we included standard logistic regression (LR) and decision trees (DT). DT refers to a binary decision tree, fitted for binary classification, and implemented in Statistics and Machine Learning Toolbox of Matlab. We also include a rule-based method which is monotone and flexible, namely the MORE algorithm for learning rule ensembles under monotonicity constraints \cite{demb_lr09}.

\begin{table}[!htb]
\centering
 \caption{\small{Classification performance in terms of mean $\pm$ standard deviation of 0/1 loss for $20\%$ (top), $50\%$ (middle), and $80\%$ (bottom) training set. Average ranks comparing significantly worse with SI at the $90 \%$ confidence level (according to a Friedman-Nemenyi test) are put in bold font.}} \label{tab2}
\scalebox{0.85}{
 \begin{tabular}{l@{\hspace{2mm}}c@{\hspace{2mm}}c@{\hspace{2mm}}c@{\hspace{2mm}}c@{\hspace{2mm}}c@{\hspace{2mm}}c}
\hline\noalign{\smallskip}
 data set & DT & LR & CR & MORE & SI \\
\noalign{\smallskip}\hline\noalign{\smallskip}\noalign{\smallskip}\noalign{\smallskip}
DGS1   & .106$\pm$.012(5)  & .095$\pm$.009(3)  & .093$\pm$.009(2)  & .098$\pm$.012(4)  & .093$\pm$.009(1)\\
DGS2   & .096$\pm$.013(5)  & .088$\pm$.008(1)  & .091$\pm$.008(3)  & .094$\pm$.007(4)  & .089$\pm$.009(2)\\
DBS    & .206$\pm$.052(5)  & .199$\pm$.060(4)  & .183$\pm$.046(3)  & .171$\pm$.041(1)  & .180$\pm$.035(2)\\
MMG    & .199$\pm$.024(5)  & .171$\pm$.011(2)  & .171$\pm$.011(3)  & .172$\pm$.010(4)  & .169$\pm$.013(1)\\
MPG    & .125$\pm$.025(5)  & .110$\pm$.019(3)  & .107$\pm$.017(2)  & .102$\pm$.015(1)  & .113$\pm$.020(4)\\
ESL    & .122$\pm$.016(5)  & .077$\pm$.013(1)  & .080$\pm$.012(2)  & .094$\pm$.011(3)  & .104$\pm$.012(4)\\
ERA    & .312$\pm$.018(5)  & .291$\pm$.012(2)  & .291$\pm$.011(1)  & .304$\pm$.016(3)  & .308$\pm$.017(4)\\
BCC    & .307$\pm$.046(5)  & .285$\pm$.028(4)  & .281$\pm$.036(3)  & .269$\pm$.025(1)  & .270$\pm$.029(2)\\
BCW    & .065$\pm$.014(5)  & .046$\pm$.014(3)  & .046$\pm$.012(2)  & .043$\pm$.010(1)  & .050$\pm$.014(4)\\
LEV    & .177$\pm$.018(5)  & .165$\pm$.010(3)  & .165$\pm$.011(2)  & .169$\pm$.017(4)  & .165$\pm$.012(1)\\
HAB    & .320$\pm$.043(5)  & .270$\pm$.025(3)  & .265$\pm$.020(2)  & .273$\pm$.016(4)  & .265$\pm$.022(1)\\
ILP    & .337$\pm$.025(5)  & .303$\pm$.019(4)  & .301$\pm$.014(3)  & .301$\pm$.023(2)  & .286$\pm$.010(1)\\
\noalign{\smallskip}\hline\noalign{\smallskip}
avg. rank   &  \textbf{5.00} & 2.75  & 2.33 & 2.67 & 2.25\\
\noalign{\smallskip}\hline\noalign{\smallskip}\noalign{\smallskip}\noalign{\smallskip}
DGS1    & .094$\pm$.012(5)  & .091$\pm$.010(4)  & .087$\pm$.010(1)  & .0898$\pm$.013(2)  & .090$\pm$.011(3)\\
DGS2    & .088$\pm$.012(5)  & .078$\pm$.010(1)  & .080$\pm$.009(2)  & .0881$\pm$.012(4)  & .085$\pm$.011(3)\\
DBS     & .181$\pm$.058(5)  & .167$\pm$.048(4)  & .161$\pm$.044(3)  & .1391$\pm$.045(1)  & .159$\pm$.036(2)\\
MMG     & .197$\pm$.018(5)  & .165$\pm$.016(2)  & .166$\pm$.015(3)  & .1671$\pm$.012(4)  & .162$\pm$.019(1)\\
MPG     & .096$\pm$.024(4)  & .095$\pm$.015(3)  & .098$\pm$.017(5)  & .0868$\pm$.016(1)  & .095$\pm$.018(2)\\
ESL     & .097$\pm$.019(4)  & .069$\pm$.014(2)  & .068$\pm$.013(1)  & .0876$\pm$.017(3)  & .105$\pm$.018(5)\\
ERA     & .302$\pm$.019(5)  & .291$\pm$.019(2)  & .290$\pm$.017(1)  & .2973$\pm$.015(3)  & .299$\pm$.015(4)\\
BCC     & .291$\pm$.034(5)  & .263$\pm$.029(3)  & .269$\pm$.032(4)  & .2459$\pm$.032(1)  & .263$\pm$.031(2)\\
BCW     & .052$\pm$.011(5)  & .037$\pm$.008(3)  & .037$\pm$.008(2)  & .0362$\pm$.010(1)  & .044$\pm$.012(4)\\
LEV     & .149$\pm$.016(1)  & .163$\pm$.012(5)  & .162$\pm$.011(3)  & .1622$\pm$.017(4)  & .159$\pm$.014(2)\\
HAB     & .316$\pm$.036(5)  & .260$\pm$.030(2)  & .252$\pm$.030(1)  & .2984$\pm$.047(4)  & .264$\pm$.025(3)\\
ILP     & .339$\pm$.024(5)  & .304$\pm$.023(4)  & .297$\pm$.021(3)  & .2904$\pm$.019(2)  & .285$\pm$.017(1)\\
\noalign{\smallskip}\hline\noalign{\smallskip}
avg. rank   & \textbf{4.50} & 2.92 & 2.42 & 2.50 & 2.67\\
\noalign{\smallskip}\hline\noalign{\smallskip}\noalign{\smallskip}\noalign{\smallskip}
DGS1    & .089$\pm$.018(4)  & .094$\pm$.018(5)  & .088$\pm$.018(3)  & .081$\pm$.016(1)  & .085$\pm$.023(2)\\
DGS2    & .086$\pm$.018(4)  & .079$\pm$.018(2)  & .081$\pm$.017(3)  & .093$\pm$.026(5)  & .076$\pm$.017(1)\\
DBS     & .166$\pm$.090(5)  & .140$\pm$.076(2)  & .142$\pm$.069(3)  & .138$\pm$.076(1)  & .155$\pm$.066(4)\\
MMG     & .198$\pm$.031(5)  & .164$\pm$.024(2)  & .164$\pm$.024(3)  & .168$\pm$.022(4)  & .157$\pm$.028(1)\\
MPG     & .086$\pm$.012(2)  & .098$\pm$.027(4)  & .099$\pm$.026(5)  & .079$\pm$.025(1)  & .092$\pm$.029(3)\\
ESL     & .083$\pm$.028(3)  & .067$\pm$.023(2)  & .063$\pm$.020(1)  & .086$\pm$.024(4)  & .104$\pm$.030(5)\\
ERA     & .297$\pm$.027(5)  & .279$\pm$.029(1)  & .288$\pm$.026(2)  & .291$\pm$.028(3)  & .297$\pm$.026(4)\\
BCC     & .296$\pm$.053(5)  & .270$\pm$.052(3)  & .252$\pm$.047(1)  & .279$\pm$.047(4)  & .266$\pm$.052(2)\\
BCW     & .049$\pm$.018(5)  & .034$\pm$.013(3)  & .034$\pm$.013(2)  & .033$\pm$.015(1)  & .041$\pm$.019(4)\\
LEV     & .143$\pm$.021(1)  & .166$\pm$.023(5)  & .161$\pm$.024(4)  & .157$\pm$.024(3)  & .156$\pm$.020(2)\\
HAB     & .317$\pm$.056(5)  & .260$\pm$.052(3)  & .260$\pm$.052(2)  & .289$\pm$.035(4)  & .258$\pm$.051(1)\\
ILP     & .332$\pm$.038(5)  & .295$\pm$.039(3)  & .304$\pm$.037(4)  & .292$\pm$.010(2)  & .282$\pm$.038(1)\\
\noalign{\smallskip}\hline\noalign{\smallskip}
avg. rank   & \textbf{4.08} & 2.92 & 2.75 & 2.75 & 2.50 \\
\noalign{\smallskip}\hline
\end{tabular}}
\end{table}

The results in terms of  misclassification rate ($0/1$ loss) and rank statistics are shown in Table \ref{tab2} for different amounts of training data. Moreover, Table \ref{tab5} provides a summary of pairwise win/loss statistics. As can be seen, the Sugeno classifier is very competitive and performs quite strongly, often even the best. There are no truly significant differences between the classifiers, however, except that decision trees are clearly outperformed. By and large, LR, CR, MORE, and SI perform the same. This result is nevertheless interesting and encouraging, as it shows that one can take advantage of the representational and algorithmic benefits of the Sugeno classifier without the need to accept a drop in performance.

\begin{table}[!htb]
\centering
 \caption{\small{Win/loss statistics (number of data sets that the first method beats the second one), for $20 \%| 50 \% |80 \%$ training data.}} \label{tab5}
\scalebox{0.95}{
 \begin{tabular}{c|c@{\hspace{2mm}}c@{\hspace{2mm}}c@{\hspace{2mm}}c@{\hspace{2mm}}c|c|c}
\hline\noalign{\smallskip}
 data set & DT & LR & CR &  MORE &  SI  & total wins & total rank\\
\noalign{\smallskip}\hline\noalign{\smallskip}\noalign{\smallskip}\noalign{\smallskip}
DT    & --  & $0 | 1 | 3$  & $0 | 2 | 2$  & $0 | 1 | 3$  & $0 | 2 | 3$         & $0 | 6 | 11$          & $5 |5| 4$\\
LR    & $12| 11 |9 $  &  -- & $3 |4 | 6$  & $7 |6  | 6 $  & $5 |5 |4 $      & $27 | 26 | 25$    & $4 | 4| 3$\\
CR    & $12 |10 |10 $  & $9 |8 |6 $  & --  &  $ 7|7 |6 $ & $4| 6 |5 $     &  $32 | 31 | 27$      & $2 |1| 2$\\
MORE    &  $12 |11 |9 $ & $5 |7 |6 $  & $5 |5 |6 $  & --  & $6 |7 |6 $ &  $28 | 29 | 27$  & $ 3 | 2| 2$\\
SI     & $12 |10 |9 $  & $ 7|7 |8 $  & $8 |6|7 $  & $6 |5 |6 $  & --         &  $33 | 28 | 30$    & $1 | 3| 1$\\
\noalign{\smallskip}\hline
\end{tabular}}
\end{table}

As we discussed before, we exploit the restriction to $k$-maxitive capacities as a means for regularization of the Sugeno classifier, i.e., to prevent overfitting effects that might be expected when fitting a fully maxitive measure. Table \ref{tab7} shows a comparison between the fully maxitive classifier and the $k$-maxitive classifier, which treats $k$ as a hyper-parameter of the learning algorithm. As can be seen, consistent improvements can indeed be achieved through regularization with $k$. Mostly, these improvements are not very big, however, which also shows that even the unregularized Sugeno classifier is not too susceptible to overfitting.

\begin{table*}[!htb]
\centering
 \caption{\small{Comparing the classification performance of fully maxitive and $k$-maxitive Sugeno integral for $20\%$, $50\%$ and $80\%$ training data: average 0/1 loss and average $k$ chosen by the SI classifier.}} \label{tab7}
\scalebox{0.85}{
 \begin{tabular}{lccccccccc}
\hline\noalign{\smallskip}
 data set & $20\%$ training & $20\%$ training  & average & $50\%$ training   & $50\%$ training   & average & $80\%$ training   & $80\%$ training & average \\
              & fully maxitive         & $k$-maxitive   & $k$     & fully maxitive        & $k$-maxitive   & $k$     & fully maxitive         & $k$-maxitive & $k$  \\
\noalign{\smallskip}\hline\noalign{\smallskip}\noalign{\smallskip}\noalign{\smallskip}
DGS1    & .094$\pm$.010  & .093$\pm$.009 & 4.02 & .090$\pm$.010 & .090$\pm$.011 & 3.86 & .088$\pm$.017  & .085$\pm$.023 & 3.57 \\
DGS2    & .088$\pm$.009  & .089$\pm$.009 & 4.08 & .085$\pm$.012 & .085$\pm$.011 & 3.84 & .080$\pm$.019  & .076$\pm$.017 & 3.58 \\
DBS     & .183$\pm$.032  & .180$\pm$.035 & 6.50 & .166$\pm$.040 & .159$\pm$.036 & 6.11 & .154$\pm$.067  & .156$\pm$.066 & 5.93 \\
MMG     & .169$\pm$.011  & .169$\pm$.013 & 4.12 & .165$\pm$.013 & .162$\pm$.013 & 3.79 & .169$\pm$.024  & .157$\pm$.028 & 3.29 \\
MPG     & .115$\pm$.021  & .113$\pm$.020 & 6.15 & .098$\pm$.016 & .095$\pm$.018 & 5.86 & .097$\pm$.031  & .092$\pm$.029 & 5.34 \\
ESL     & .105$\pm$.012  & .104$\pm$.012 & 3.88 & .107$\pm$.016 & .105$\pm$.018 & 3.59 & .108$\pm$.027  & .104$\pm$.030 & 3.00 \\
ERA     & .308$\pm$.019  & .308$\pm$.017 & 3.47 & .300$\pm$.016 & .299$\pm$.015 & 3.56 & .303$\pm$.028  & .297$\pm$.026 & 3.16 \\
BCC     & .275$\pm$.037  & .270$\pm$.029 & 5.35 & .269$\pm$.028 & .263$\pm$.031 & 5.79 & .283$\pm$.048  & .266$\pm$.052 & 5.70 \\
BCW     & .053$\pm$.016  & .050$\pm$.014 & 7.49 & .044$\pm$.013 & .044$\pm$.012 & 6.88 & .039$\pm$.016  & .041$\pm$.019 & 6.81 \\
LEV     & .163$\pm$.011  & .165$\pm$.012 & 3.08 & .158$\pm$.012 & .159$\pm$.014 & 2.96 & .161$\pm$.024  & .156$\pm$.020 & 2.75 \\
HAB     & .290$\pm$.031  & .265$\pm$.022 & 1.82 & .293$\pm$.031 & .264$\pm$.025 & 1.73 & .277$\pm$.055  & .258$\pm$.051 & 1.52 \\
ILP     & .314$\pm$.024  & .286$\pm$.010 & 3.76 & .308$\pm$.021 & .285$\pm$.017 & 3.46 & .300$\pm$.037  & .282$\pm$.038 & 3.41\\
\noalign{\smallskip}\hline
\end{tabular}}
\end{table*}

\section{Conclusion}
\label{Conclusion}

In this paper, we proposed a novel method for binary classification that builds upon the Sugeno integral as a means for aggregating feature information in supervised machine learning. Due to the specific properties of the Sugeno integral, this approach is especially suitable for learning monotone models from ordinal data, although it can of course also be applied to learning from numerical data. We analyzed theoretical properties of the Sugeno classifier, proposed a learning algorithm based on linear programming, and assessed the performance of the classifier in an experimental study. 

Our empirical results are promising and show that the Sugeno classifier is competitive in terms of predictive accuracy, in spite of its seemingly restricted expressiveness compared to more powerful models like the (numerical) Choquet integral. Combined with its ``symbolic'' nature, this makes it quite an appealing approach to data-driven model construction, especially from the point of view of interpretable machine learning and explainable AI \cite{alon_io15,guo_ai19}. There are various directions of future work:

\begin{itemize}
\item Our approach is limited to binary classification and should be extended toward other machine learning problems, such as multinomial and ordinal classification (akin to the extension of the Choquet classifier to ordinal classification \cite{TehraniH13}).
\item Even if the true dependency between the predictor variables and the target is monotone, the training data may violate monotonicity (e.g., due to noise and errors). Nevertheless, our learning algorithm enforces monotonicity by solving a constrained otimization problem. Another idea, which has been put forward in the literature on monotonic classification \cite{feel_mr10}, is to ``monotonize'' the training data first, and then to fit an unconstrained model to the pre-processed data. This approach appears to be a viable alternative for the Sugeno classifier, too, and is certainly worth an investigation.  

\item 
Last but not the least, it is tempting to further elaborate on the connection between the Sugeno classifier and other types of classifiers. As we pointed out, by specifying the underlying capacity in a suitable way, well-known models like the $k$-of-$m$ classifier can be obtained as special cases of the Sugeno classifier. Thus, our approach may provide a unifying framework of a broader class of classifiers, suggest new methods for training such classifiers, reveal interesting relationships between them, and perhaps even suggest new classifiers as specific instantiations.
\end{itemize}

\begin{appendix}

\section{VC Dimension of the Sugeno Classifier}\label{sec:appa}

\begin{theorem}
The VC dimension of the hypothesis space $\mathcal{H}$ comprised of threshold classifiers of the form (\ref{4}) grows asymptotically at least as fast as $2^m / \sqrt{m}$.
\end{theorem}

\begin{proof}
We show the model class $\mathcal{H}$ can shatter a sufficiently large data set $\mathcal{D}$ of the size $2^m / \sqrt{m}$. We construct the set $\mathcal{D}$ by using the binary attribute values, which means that $x_i \in \{0, 1\}$ for all $1 \leq i \leq m$. Accordingly, each instance $\vec{x} = (x_1,\ldots , x_m) \in \{0, 1\}^m$ can be identified with a subset of indices $S_{\vec{x}} \subseteq X$, namely its indicator set $S_{\vec{x}} = \{i| x_i = 1\}$. 

We recall a well-known result of Sperner \cite{Sperner}, who showed that the maximum cardinality of any antichains of the set $[m]$ is the so-called Sperner number $\binom{m}{\lfloor m/2 \rfloor}$. Antichain is a notion in combinatorics regarding the incomparable subsets of an arbitrary set. An antichain $\mathcal{A}$ of  the set $[m]$ is a nonempty proper subset of $2^{[m]}$ such that, for all $A,B \in \mathcal{A}$, neither $A \subseteq B$ nor $B \subseteq A$. The Sperner number is obviously restricted due to the above non-inclusion constraint on pairs of subsets. Sperner showed that the corresponding antichain $\mathcal{A}$ is given by the family of all $q$-subsets of $X$ with $q = \lfloor m/2 \rfloor$, that is, all subsets $A \subseteq X$ such that $|A| = q$.

From a decision making perspective, each attribute can be interpreted as a criterion. Using this fact, we are going to define the specific dataset $\mathcal{D}$ in terms of the collection of all instances $\vec{x} = (x_1, \ldots , x_m) \in \{0, 1\}^m$ whose indicator set $S_{\vec{x}}$ is a $q$-subset of $X$. Naturally, the instances in $\mathcal{D}$ are therefore maximally incomparable; indeed, each instance in $\mathcal{D}$ satisfies exactly $q$ of the $m$ criteria, and there is not a single “dominance” relation in the sense that the set of criteria satisfied by one instance is a superset of those satisfied by another instance. This substantial property will help us to show that $\mathcal{D}$ can be shattered by $\mathcal{H}$.

It should be noted that the set $\mathcal{D}$ can be shattered by a model class $\mathcal{H}$ if, for each subset $\mathcal{P} \subseteq \mathcal{D}$, there is a model $H \in \mathcal{H}$ such that $H(\vec{x}) = 1$ for all $\vec{x} \in \mathcal{P}$ and $H(\vec{x}) = 0$ for all $(\vec{x}, y) \in \mathcal{D} \setminus \mathcal{P}$. Now, we consider any such subset $\mathcal{P}$ from the data set $\mathcal{D}$ as constructed above, and the Segeno integral given by \eqref{2}. We define the measure $\mu$ as follows:
$$\mu(E)= \begin{cases} 1 & \text{if}~ E \supseteq S_{\vec{x}}~ \text{for some}~ \vec{x} \in \mathcal{P}, \\
									      0 & \text{otherwise}.
\end{cases}$$
Obviously, this definition is feasible and yields a proper capacity $\mu$.

Now, consider expression \eqref{2} for any $(\vec{x}, y) \in \mathcal{P}$. For $E = S_{\vec{x}}$, we have $\min_{i \in E} x_i = 1$ and $\mu(E) = 1$, and hence $\text{S}_{\mu}(\vec{x}) = 1$. Likewise, consider expression \eqref{2} for any $(\vec{x}^{'}, y^{'}) \in \mathcal{D} \setminus \mathcal{P}$. From the construction of $\mu$ and the fact that, for each pair $\vec{x} \neq \vec{x}^{'}$ in training set, neither $S_{\vec{x}} \subseteq S_{\vec{x}^{'}}$ nor $S_{\vec{x}^{'}} \subseteq S_{\vec{x}}$, it follows that either $\min_{i \in E} x_{i}^{'} = 0$ or $\mu(E) = 0$, regardless of $E$. More specifically, we can distinguish three cases according to the size of E: If $|E| < q$, then $\mu(E) = 0$. If $|E| > q$, then $\min_{i \in E} x_{i}^{'} = 0$, because $S_{\vec{x}^{'}}$ is a $q$-subset of $X$. If $|E| = q$, then either $E = S_{\vec{x}}$ for some $\vec{x} \in \mathcal{P}$, in which case $\min_{i \in E} x_{i}^{'} = 0$, or $E \neq S_{\vec{x}}$ for all $\vec{x} \in \mathcal{P}$, in which case $\mu(E) = 0$. Consequently, the Sugeno integral is given as follows:
$$\text{SI}_{\mu}(\vec{x})= \begin{cases} 1 & \vec{x} \in \mathcal{P}, \\
									      0 & \text{otherwise}.
\end{cases}$$
Thus, with $\beta=1/2$,  the classifier \eqref{4} behaves exactly as required, that is, it classifies all $\vec{x} \in \mathcal{P}$ as positive and all $\vec{x} \notin \mathcal{P}$ as negative.

We make use of Sterling's approximation of large factorials (and hence binomial coefficients) for the asymptotic analysis. For the sequence $(b_1, b_2, \ldots)$ of the so-called central binomial coefficients $b_n$, it is known that
\begin{align*}
b_n=\binom{2n}{n}=\frac{(2n)!}{(n!)^2} \geq \frac{1}{2} \frac{4^n}{\sqrt{\pi \cdot n}}.
\end{align*}
By setting $n=m/2$ and ignoring constant terms, we conclude that the VC dimension of the model class $\mathcal{H}$ grows asymptotically at least as fast as $2^m / \sqrt{m}$.
\end{proof}

\section{Determining the Direction of Features}
\label{sec:appb}


The Sugeno classifier assumes monotonicity in the sense that the dependency between predictor and target variables is either ``the higher the better'' or ``the lower the better''. In many cases, the direction is known beforehand and provided to the learner as part of the prior knowledge. Otherwise, if this is not the case, we suggest to determine the direction as a pre-processing step in a data-driven way as follows:
\begin{itemize}
\item We fit a flexible model, such as a neural network, to the training data. This model does neither require nor impose any monotonicity condition. 
\item For every variable $x_j$, we determine the output of the model for each original training instance $\vec{x}^{(i)} = (x_{1}^{(i)}, \ldots, x_{m}^{(i)})$ and for the same instance with a slightly increased value for $x_j$, i.e., with $x_{j}^{(i)}$ replaced by $x_{j}^{(i)}+\delta$.
\item We count the number of cases in which the output of the model decreases and increases, respectively, and have a guess on the direction of the variable's influence depending on which of the cases prevails. 
\end{itemize}

\section{Data Sets}
\label{appdata}

For a description of the following data sets, we refer to \cite{fallhullermei12}: Bosch (DBS), Mammographic (MMG), Auto MPG, Employee Selection (ESL), Employee Rejection/Acceptance (ERA), Employee Rejection/Acceptance (ERA), Breast Cancer (BCC), Breast Cancer Wisconsin (BCW), Lecturers Evaluation (LEV), Haberman's Survival Data (HAB), Indian Liver Patient (ILP).


The Dagstuhl-15512 ArgQuality Corpus (DGS) is a data set that consists of 25 to 35 textual debate portal arguments for two postures on 16 issues, such as \textit{christianity vs.\ atheism} and \textit{is the school uniform a good or bad idea}. For each posture pair, the five first-rate texts are taken and also five further are chosen via stratified sampling. So, both high-level arguments and different lower-level qualities are covered and finally, 320 texts (20 argumentative comments $\times$ 16 issues) are chosen. Then, three annotators (two females, one male, from three countries, who work at two universities and one company) discuss all 320 texts. They assess 15 quality dimensions (including its overall quality) in the taxonomy for each comment using three ordinal scores: 1 (low), 2 (average) and 3 (high). In our experiments, we consider the overall quality as the target attribute and binarize it by distinguishing between low-level argument (score 1) and high-level argument (scores 2 and 3). In addition, to improve the accuracy of the classification and predicting based on more related subgroup of attributes, we distinguish two subsets of the attributes, one related to emotional and relevancy criteria, and one related to the logical and reasonableness criteria. The first group includes \textit{local acceptability, cogency, effectiveness, global relevance}, and \textit{emotional appeal}. The second group consists of \textit{arrangement, cogency, global sufficiency, reasonableness, credibility}, and \textit{sufficiency}.

\end{appendix}



\end{document}